\documentclass{article}

\usepackage{arxiv}
\usepackage[numbers]{natbib}

\usepackage[utf8]{inputenc} 
\usepackage[T1]{fontenc}    
\usepackage{hyperref}       
\usepackage{url}            
\usepackage{booktabs}       
\usepackage{amsfonts}       
\usepackage{nicefrac}       
\usepackage{microtype}      
\usepackage{verbatim}
\usepackage{graphicx}

\usepackage{doi}
\usepackage{amsmath}
\usepackage{amssymb}  
\newcommand\norm[1]{\left\lVert#1\right\rVert}
\usepackage{amsthm}
\newtheorem{theorem}{Theorem}[section]
\newtheorem{remark}{Remark}[section]

\newtheorem{definition}{Definition}[section]
\usepackage[inline]{enumitem}
\usepackage{tikz}
\usetikzlibrary{positioning,fit,calc,arrows,shapes,backgrounds}
\pgfdeclarelayer{background}
\pgfsetlayers{background,main}

\usepackage{cleveref}       

\title{Cooperative Collision Avoidance in Mobile Robots using Dynamic Vortex Potential Fields}

\author{
  Wayne Paul Martis \\
  Technische Universit\"{a}t Dortmund \\
  Dortmund, Germany\\
  \texttt{waynepaul.martis@tu-dortmund.de} \\
   \And
  Sachit Rao \\
  International Institute of Information Technology Bangalore \\
  Bangalore, India\\
  \texttt{sachit@iiitb.ac.in}
}

\renewcommand{\shorttitle}{Vortex Potential Fields for Collision Avoidance}

\begin{document}
\maketitle

\begin{abstract}
In this paper, the collision avoidance problem for non-holonomic robots moving at constant linear speeds in the 2-D plane is considered. The maneuvers to avoid collisions are designed using dynamic vortex potential fields (PFs) and their negative gradients; this formulation leads to a reciprocal behaviour between the robots, denoted as being cooperative. The repulsive field is selected as a function of the velocity and position of a robot relative to another and introducing vorticity in its definition guarantees the absence of local minima. Such a repulsive field is activated by a robot only when it is on a collision path with other mobile robots or stationary obstacles. By analysing the kinematics-based engagement dynamics in polar coordinates, it is shown that a cooperative robot is able to avoid collisions with non-cooperating robots, such as stationary and constant velocity robots, as well as those actively seeking to collide with it. Conditions on the PF parameters are identified that ensure collision avoidance for all cases. Experimental results acquired using a mobile robot platform support the theoretical contributions. 
\end{abstract}


\section{Introduction}

Collision avoidance (CA) is a crucial element in many applications involving mobile robots as well as aerial vehicles. For example, in the deployment of a swarm of robots for area exploration--see \cite{Gayle2009,yan2013survey,McIntyre2016} for applications and challenges--or air traffic management, where several aerial robots sharing the same airspace have to reach their respective goal locations while avoiding collisions with each other, \cite{Kosecka1997,Du2019}. A commonly applied technique to avoid collisions in such dynamic environments is based on artificial potential fields (APFs), mainly because this technique uses local information and is easy to implement. As is well known, in this approach, an attractive field is defined at the goal location and repulsive fields are defined around obstacles. These fields are typically functions of positions and/or the velocities of the robots and obstacles. Calculating the negative gradient of these PF functions leads to expressions for the inputs to the robot that can steer it towards its goal.

The choice of PFs, however, can lead to the presence of local minima, where, a robot gets ``trapped'' and is unable to move towards its goal. Variants of PF functions have been proposed that avoid this issue, one of them being the class of harmonic PFs, denoted by $\mathcal{U}$, that satisfy the Laplacian equation $\nabla^T\nabla \mathcal{U}$=0, \cite{Guldner1997}. Alternative PFs have been proposed in \cite{Urakubo2018,Du2019,Keyu2020}.

Another variant is the use of the so-called curl-free vector field, also the vortex field, as the repulsive field defined around the obstacles rotates around the obstacle. The direction of rotation can be selected by the appropriate choice of signs of the gradients. Such fields have been used in \cite{DeLuca1994,Kosecka1997,Shibata2014,McIntyre2016,Choi2020} to design repulsive PFs for collision avoidance. In our work, where we focus on cooperative robots, we select such a rotating repulsive field to ensure that a pair of robots on a collision path turn in the same direction and avoid collisions; this formulation holds even for multiple robots. As described in \cite{Kosecka1997}, fixing the direction of rotation of the vortex repulsive field corresponds to fixing the ``rule of the road''. Repulsive PFs akin to the vortex PF we use are designed in \cite{Garcia-Delgado2015}, where the force vectors are chosen to lie tangentially to an obstacle using a rotational matrix, and in \cite{Fedele2018}, where a helicoidal PF function is proposed around an obstacle that enables a robot to rotate around it. It is emphasised these designs do not explicitly consider the relative position or the velocity.

We use the dynamic vortex field function proposed in \cite{Park2008} to design the repulsive field, where, the field is a function of the position and velocity of a robot \textit{relative} to another robot. As the repulsive field is defined in terms of relative velocity and position, the direction of rotation of the field becomes fixed, which in turn, ensures that robots turn in the same direction. It should be mentioned that the dynamic vortex field in \cite{Park2008} is designed for a robotic manipulator to avoid dynamic obstacles, which we apply to mobile robots, thus bringing in the novelty to our approach.

Indeed, the idea of requiring cooperative robots to turn in the same direction is the essence of the Hybrid Reciprocal Velocity Obstacle (HRVO) algorithm, \cite{Snape2011}, as well as the Optimal Reciprocal Collision Avoidance (ORCA) algorithm, \cite{vandenBerg2011,Alonso-Mora2013}. In these algorithm, cooperative robots perform maneuvers that are reciprocal to each other to avoid collisions. The maneuvers are designed in the velocity space of the robots; for example, in the HRVO algorithm, a robot computes a new velocity that lies outside the velocity obstacle induced by other robots and which is also closest to its preferred velocity. In contrast to these algorithms, what we propose does not require the computation of velocity obstacles and reciprocity appears as a natural consequence of the formulation of the repulsive field. We also consider non-point robots, modelled as circles, and determine conditions on the PF parameters as functions of the robots' radii, initial separation distance, as well as their relative velocities that guarantee collision avoidance. We explicitly consider a case where a robot is actively seeking to collide, denoted as the attacking robot, with a cooperative robot and even for this case, bounds on the PF parameters that enable the cooperative robot to repel the attacking robot are identified. The attacking robot is assumed to apply the standard attractive PF to collide with the cooperative robot. 

The main contributions of this paper are the application of a dynamic vortex potential field for collision avoidance in cooperative robots
\begin{enumerate}
    \item that naturally introduces reciprocity between cooperative robots, thus eliminating the need to state rules that define the directions of turn;
    \item against non-cooperative robots that are either stationary or move at a constant velocity;
		\item that can be applied to repel robots that actively seek to collide with cooperative robots; and
		\item that consider the radius of the circle that bounds each cooperative robot.
\end{enumerate}
It is also highlighted that the paper presents rigorous proofs of collision avoidance for all cases, using Lyapunov theory and results from aerospace guidance literature. The engagement between robots, in all cases, is defined using kinematics-based dynamics in polar coordinates. The proofs are based on the conditions--identified in \cite{Chakravarthy1998}--that are both necessary and sufficient for collision of points moving in 2-D space. Collision avoidance is demonstrated by proving that the collision conditions form unstable equilibrium points and hence, are \emph{never} satisfied. It is assumed that all types of robots are homogeneous and move at the same linear speed. However, since the algorithm is developed in the relative velocity space, it is applicable to robots that move at different linear speeds.

We present experimental results for all our theoretical contributions in an indoor robotics platform that consists of several differential-drive robots that satisfy the non-holonomic constraints. As will be discussed, the theoretical results can be applied directly with minor tuning of the PF parameters, primarily to account for the inertial parameters of the robots that are not considered in the theoretical derivations. 

The paper is organised as follows: The kinematic model of the robot as well as the engagement dynamics between two robots on a collision path are presented in Sec.~\ref{Sec:Prelim}. The attractive and dynamic vortex repulsive fields that are used to design the robot maneuvers are described in Sec.~\ref{Sec:MainRes}; the proofs of collision avoidance are also presented here. Experimental results on the application of the proposed approach on differential drive robots are presented in Sec.~\ref{Sec:ExpRes}, followed by concluding remarks in Sec.~\ref{Sec:Conc}.

\section{Preliminaries}\label{Sec:Prelim}

The kinematics of the robot as well as that of the engagement dynamics defined in terms of relative velocities in polar coordinates are described. These relative velocities will be used to prove collision avoidance.

\subsection{Robot description}

A swarm of $N\geq 2$ homogeneous robots that move at constant linear speeds, $V$, is considered. Each robot is assumed to be a point-mass and non-holonomic, whose movements in a Cartesian plane ($X-Y$) are given by the kinematic relations
\begin{subequations}\label{RobotDyn}
\begin{align}
    \dot{x}_{Ri} &= V\cos{\left(\phi_{Ri}\right)}, \ \dot{y}_{Ri} = V\sin{\left(\phi_{Ri}\right)},\label{xyDyn}\\
    \dot{\phi}_{Ri} &= \omega_{Ri} \label{phiDyn} 
\end{align}
\end{subequations}
where, $(x,y)_{Ri}$ are the coordinates of robots $R_i, \ i=1,\cdots,N$, in the Cartesian space with a known origin; $\phi_{Ri}$ is the heading angle, which is measured positive counter-clockwise about the $X-$axis; and $\omega_{Ri}$ is the angular velocity, which can be varied. This variation leads to the accelerations
\begin{subequations}\label{RobotDynAccl}
\begin{align}
    \ddot{x}_{Ri} &= -V\sin{\left(\phi_{Ri}\right)}\dot{\phi}_{Ri}=F_{xRi}\label{FxDef} \\
    \ddot{y}_{Ri} &= V\cos{\left(\phi_{Ri}\right)}\dot{\phi}_{Ri}=F_{yRi},\label{FyDef}
\end{align}
\end{subequations}
where, the terms $F_{xRi},F_{yRi}$ can be interpreted as forces acting on the robot that steer the robot to known goal locations $(x,y)_{RiG}$. It is these forces that are selected as the negative gradients of the PFs defined in Sec.~\ref{Sec:MainRes}.

\subsection{Engagement Dynamics}\label{Sec:EngDyn}

\begin{figure}[hbt!]
	\centering
	\resizebox{0.4\columnwidth}{!}{		
		\begin{tikzpicture}
		
			
			\node[fill=gray!30,thick,circle,inner sep=0pt, minimum size = 0.5cm] (R1) at (0, 0) {$R_1$};
			\node[fill=gray!30,thick,circle,inner sep=0pt, minimum size = 0.5cm] (R2) at (2, 2) {$R_2$};
			
			\draw[thick](R1) -- (R2) node[pos=0.7,below] {$r$}; 
			
			\draw[->,dashed] (R1) -- (1.75,0) node[pos=1,below] {$X$}; 
			\draw[->,dashed] (R1) -- (0,1.25) node[pos=0.8,left] {$Y$};
			
			\draw[->,thick] (0.5,0) arc (0:45:0.5cm); 
			\draw (1.15,0.2) node[left] {$\theta_{1}$}; 
			
			\draw[->,dashed] (1.25,2) -- (3.25,2);
			\draw[->,thick] (2.5,2) arc (0:225:0.5cm); 
			\draw (1.58,1.7) node[left] {$\theta_{2}$}; 
			
			\draw[->,thick](R1) -- (0.75,1.29) node[pos=0.7,left] {$V$}; 
			\draw[->,thick](R2) -- (0.71,2.75) node[pos=0.5,below left] {$V$}; 
			
			\draw[->,thick] (1.1,0) arc (0:60:1.1cm); 
			\draw (1.15,0.2) node[right] {$\phi_{R1}$}; 
			
			\draw[->,thick] (3,2) arc (0:150:1cm); 
			\draw (3.05,2.2) node[right] {$\phi_{R2}$}; 
			
		\end{tikzpicture}
	}
	\caption{A pair of robots on a collision path} 
	\label{fig:2RobClsSchematic}
\end{figure}
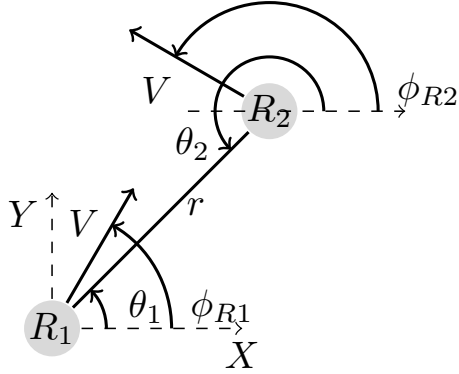

Consider 2 robots, $R_1$ and $R_2$, moving in the plane at constant linear speeds, as shown in Fig.~\ref{fig:2RobClsSchematic}. Following the notations used in aerospace guidance literature, denote the separation distance between them as
\begin{equation}\label{rDef}
    r_{1}=\sqrt{\left(x_{R1}-x_{R2}\right)^2 + \left(y_{R1}-y_{R2}\right)^2} = r_2,
\end{equation} 
and the line-of-sight (LOS) angles as $\theta_{i}$, which are computed as
\begin{equation}\label{LOSAngDef}
    \theta_{i}=\tan^{-1}{\left(\frac{y_{Rj}-y_{Ri}}{x_{Rj}-x_{Ri}}\right)}, \ i,j=1,2, \ i\neq j,
\end{equation}
and are the angles made by the line joining the two robots and the $X-$axis passing through the robots, respectively.

Using \eqref{RobotDyn}, \eqref{rDef}, and \eqref{LOSAngDef}, define the relative velocities in polar coordinates
\begin{subequations}\label{VrVthExp}
	\begin{align}
		\dot{r}_i&= V_{ri} =  V\cos{\left(\phi_{Rj}-\theta_{i}\right)} - V\cos{\left(\phi_{Ri}-\theta_{i}\right)} \label{VriExp} \\
		r_i\dot{\theta}_{i} &= V_{\theta i} = V\sin{\left(\phi_{Rj}-\theta_{i}\right)} - V\sin{\left(\phi_{Ri}-\theta_{i}\right)}, \label{VthiExp}
	\end{align}
\end{subequations}
which, along with \eqref{RobotDynAccl}, yield the relative accelerations
\begin{subequations}\label{VrVthDotExp}
	\begin{align}
		\dot{V}_{ri} &= \frac{V_{\theta i}^2}{r_i} +  \cos{\left(\theta_{i}\right)}\left(F_{xRj}-F_{xRi}\right) + \sin{\left(\theta_{i}\right)}\left(F_{yRj}-F_{yRi}\right) \label{VriDotExp} \\
		\dot{V}_{\theta i} &= \frac{-V_{ri}V_{\theta i}}{r_i} -  \sin{\left(\theta_{i}\right)}\left(F_{xRj}-F_{xRi}\right) + \cos{\left(\theta_{i}\right)}\left(F_{yRj}-F_{yRi}\right).\label{VthiDotExp}
	\end{align}
\end{subequations}
Since the LOS angles satisfy $\theta_{2}=\pi+\theta_{1}$, the relative velocities also satisfy $V_{r1}=V_{r2}$ and $V_{\theta 1}=V_{\theta 2}$. Thus, when a pair of robots are considered, the subscript $i$ is dropped for the relative velocity and acceleration terms.

As proved in \cite[Lemma 2]{Chakravarthy1998}, the robots $R_1$ and $R_2$ moving at a constant velocity are on a collision path if $V_{r}<0$ \textit{and} $V_{\theta}=0$; these are both necessary and sufficient conditions for collision. Thus, if the forces $F_{xRi},F_{yRi}$ acting on each robot are selected such that these conditions are \emph{violated}, then, the robots avoid a collision with each other. These forces are chosen based on the negative gradients of the dynamic vortex potential field--which forms the repulsive field around each robot considered as an obstacle--and an attractive field defined at the goal location; these designs form the main results of the paper, which are presented next.

\section{Main Results}\label{Sec:MainRes}

To design the inputs that steer each robot to its respective goal location and avoid collisions, the following standard assumptions are made:
\begin{enumerate}[label=A.\arabic*]
	\item The goal locations, $(x,y)_{RiG}$, of each robot are stationary and known.
	\item Each robot possesses information on the positions and velocities of other robots in its vicinity.
\end{enumerate}


\subsection{Attractive Potential Field}\label{Sec:AttPF}

Since each robot, $R_i$ knows its goal location, the inputs that steer the robot towards its goal are designed based on the attractive PF function
\begin{equation}\label{UAttPF}
    \mathcal{U}_{\text{Att}i} = \kappa\sqrt{\left(x_{Ri}-x_{RiG}\right)^2 + \left(y_{Ri}-y_{RiG}\right)^2},
\end{equation}
where, $\kappa>0$ is a user-defined parameter. As can be seen, $\mathcal{U}_{\text{Att}i}$ is simply a scaled Euclidean distance between the current location of the robot and the goal. Now, it is shown that if the robot inputs are chosen based on the negative gradient of $\mathcal{U}_{\text{Att}i}$, then, it is attracted towards the goal. The proof is based on constructing a Lyapunov function and analysing its dynamics. 

For simplicity, consider a single robot and an obstacle-free environment (the subscript $i$ is therefore dropped). Thus, using the terms introduced in Sec.~\ref{Sec:EngDyn}, the forces generated by the attractive field become
\begin{equation}\label{FxyAtt}
    F_{xR\text{Att}} = \frac{-\partial \mathcal{U}_{\text{Att}}}{\partial\left(x_{R}-x_{RG}\right)}=+\kappa \cos{\left(\theta\right)}, \ F_{yR\text{Att}} = \frac{-\partial \mathcal{U}_{\text{Att}}}{\partial\left(y_{R}-y_{RG}\right)}=+\kappa \sin{\left(\theta\right)}.
\end{equation}
In deriving these partial derivatives, the expressions for the LOS angle
\begin{equation}\label{LOSAngExp}
	\cos{\left(\theta\right)} = \frac{x_{RG}-x_{R}}{r}, \ \sin{\left(\theta\right)} = \frac{y_{RG}-y_{R}}{r} 
\end{equation}
are used, where, $r$ is the separation distance between the robot and the goal. Substituting \eqref{FxyAtt} in \eqref{VrVthDotExp}, the relative accelerations become
\begin{equation}\label{VrVthDotAtt}
	\dot{V}_{r} = \frac{V_{\theta}^2}{r} - \kappa, \ \dot{V}_{\theta} = \frac{-V_{r}V_{\theta}}{r}.     
\end{equation}
Also, since the goal location, or target, is stationary, the relative velocities, $V_{r},V_{\theta}$, in \eqref{VrVthExp}, reduce to $V_r = - V\cos{\left(\phi_{R}-\theta\right)}, \ V_{\theta}=- V\sin{\left(\phi_{R}-\theta\right)}$, 
from which it emerges that
\begin{equation}\label{VrVthCirc}
    V_r^2+V_{\theta}^2=V^2.
\end{equation}
This condition implies that in the $V_{r}-V_{\theta}$ space, the relative velocities are constrained to lie on a circle with the centre at the origin and radius $V$.

The robot is attracted towards the target if the collision conditions $V_r<0$ and $V_{\theta}=0$ are satisfied. Since these states have to satisfy \eqref{VrVthCirc}, for collision to occur, the relative velocity $V_r$ has to satisfy $V_r=-V$. Also note that, since $|V_r|\leq V$, even if $\dot{V}_{r} =- \kappa$ for $V_{\theta}=0$ and $r \neq 0$, in the dynamics \eqref{VrVthDotAtt}, the relative velocity satisfies $V_{r}=-V$. Thus, to examine the occurrence of collision, by defining $\tilde{V}_r=V_r+V$, the stability of the equilibrium point $r,V_{\theta},\tilde{V}_r=0$ is analysed, following Lyapunov stability concepts, as discussed in \cite{Khalil2002}. Consider the Lyapunov function
\begin{equation}\label{LyapFnAtt}
    \mathcal{V}_{LAtt}=\kappa r + \frac{1}{2}V_{\theta}^2 + \frac{1}{2}\tilde{V}_r^2,
\end{equation}
which satisfies $\mathcal{V}_{LAtt}=0$ at the equilibrium point. By expressing the dynamics \eqref{VrVthDotAtt} in terms of $\tilde{V}_r$, the derivative of $\mathcal{V}_{LAtt}$ calculated using these dynamics is given by
\begin{equation}\label{LyapFnAttDot}
	\dot{\mathcal{V}}_{LAtt} = -V\left(\kappa - \frac{V_{\theta}^2}{r}\right).
\end{equation}
Now, since $|V_{\theta}|\leq V$, by choosing $\kappa>(V^2/r^*)$, where, $r^*>0$, the derivative $\dot{\mathcal{V}}_{LAtt}<0 \ \forall \ r>r^*$ and $V_{\theta}\neq0$. Thus, the trajectories of the non-linear dynamics are attracted towards $r=r^*$ if the initial distance $r_0>r^*$. Note that, from \eqref{FxyAtt}, $\kappa$ is a measure of the acceleration that can be provided by the robot. 

While the result \eqref{LyapFnAttDot} suggests a limit cycle type behaviour in the vicinity of $r^*$, the analysis is further extended by determining if there exist initial conditions and a time instant, say $t=t_1>0$, when $V_{\theta}(t_1)=0$, $V_r(t_1)=-V \Rightarrow \tilde{V}_r=0$, and $r(t_1)\neq0$ hold. If such conditions can be identified, then, the robot reaches the equilibrium point within a finite time interval, say $t_1<t_2<\infty$, since $\dot{\mathcal{V}}_{LAtt}=-V\kappa<0 \ \forall \ t\geq t_1$. In addition, $\forall \ t\geq t_2$, though $V_r=-V$ and $\dot{V}_{r} =- \kappa$ hold, since $|V_r|\leq V$, the magnitude of $V_r$ cannot increase any further, implies that these conditions lead to the robot being ``trapped'' at the target location.

To determine if the conditions mentioned can be found, consider the dynamics of $V_r$, in \eqref{VrVthDotAtt}, and its simplification $\dot{V}_{r}=-\beta \kappa$, where, $\beta<1$. Note that if $r_0<r^*$ and $V_{\theta}(t=0)\neq0$, then $\beta<0$ as well; however, from \eqref{LyapFnAttDot}, since these conditions lead to $\dot{\mathcal{V}}_{LAtt}>0$, the robot moves away from the target resulting in $r>r^*$. Thus, the analysis can be performed by assuming an initial separation distance $r_0>r^*$ and $0<\beta<1$. Now, by integrating $\dot{V}_{r}=-\beta \kappa$ once, with the initial condition $V_{r0}$, leads to $V_{r}(t_1)=V_{r0}-kt_1$ and the value of $t_1$ when $V_{r}(t_1)=-V$ and $V_{\theta}(t_1)=0$, to be $t_1=\dfrac{V_{r0}+V}{\beta \kappa}$. Integrating the dynamics of $\dot{r}=V_r$, results in
\begin{equation}\label{rt1Val}
	r(t_1) = r_0 + \frac{\left(V_{r0}-V\right)^2}{2\beta \kappa}\Rightarrow r(t_1)\geq r_0>r^*.
\end{equation}
Thus, for $t\geq t_1$, since $\dot{\mathcal{V}}<0$ and the conditions for collision are satisfied, the robot reaches its target location for any initial conditions $V_[r0],V_{\theta0}$, and $r_0$, implying that the equilibrium point $r,V_{\theta},\tilde{V}_r=0$ is globally asymptotically stable.

The existence of $V_{\theta}(t_1)=0$, $V_r(t_1)=-V$, and $r(t_1)\neq0$ can be interpreted as follows: based on the initial conditions, the robot moves away from the target location, but changes its orientation such that it points towards the target and then moves to the target at constant speed.

\begin{remark}\label{AttProofRemark}
The presented analysis is required mainly because the robot does not decelerate as it approaches the target, but moves at a constant linear speed. This is a standard setting in aerospace guidance literature, where, an interceptor is considered to move at constant speed and the forces acting on the interceptor are designed to point it towards its target, leading to a collision. In the current robotics application, however, in a practical scenario, while implementing the forces \eqref{FxyAtt}, the robot's speed can be switched off to zero, when it is close (with allowable deviations) to the target location.
\end{remark}

The robot inputs that allow it to maneuver around various types of obstacles, both static and dynamic, are designed next.

\subsection{Dynamic Vortex Repulsive Potential Field}\label{Sec:RepPF}

Consider two robots $R_i$ and $R_j$, as shown in Fig.~\ref{fig:2RobClsSchematic}, and let their velocities in the plane be such that they are on a collision path. For each robot, the repulsive field, as proposed in \cite{Park2008}, is selected in the form
\begin{subequations}\label{UrepDef}
\begin{align}
    \mathcal{U}_{\text{Rep}i} &= \begin{cases}
                                    \lambda \left(\cos{\gamma_i}\right)^2\frac{\norm{\mathbf{v}_{ij}}}{r_i} \ & \text{if } \frac{\pi}{2}<\gamma\leq \pi \\
                                    0 & \ \text{otherwise}
                                \end{cases}, \ \lambda>0, \label{URepEq} \\
    \intertext{where,}
    \cos{\gamma_i} &= \frac{\mathbf{v}_{ij}^T\mathbf{x}_{ij}}{\norm{\mathbf{v}_{ij}}r_i}, \ \mathbf{x}_{ij} = \begin{bmatrix} \left(x_{Rj}-x_{Ri}\right) \\ \left(y_{Rj}-y_{Ri}\right) \end{bmatrix}, \ \mathbf{v}_{ij} = \begin{bmatrix} \left(V\cos{\left(\phi_{Rj}\right)}-V\cos{\left(\phi_{Ri}\right)}\right) \\ \left(V\sin{\left(\phi_{Rj}\right)}-V\sin{\left(\phi_{Ri}\right)}\right) \end{bmatrix}.  \label{cosgammaExp}
\end{align}
\end{subequations}
As can be observed, the vectors $\mathbf{v}_{ij}$ and $\mathbf{x}_{ij}$ are the velocities and positions of robot $R_j$ relative to robot $R_i$. The term $\cos{\gamma_i}$ in \eqref{cosgammaExp} has a direct connection with the relative velocity terms introduced in Sec.~\ref{Sec:EngDyn}. By computing the dot product $\mathbf{v}_{ij}^T\mathbf{x}_{ij}$, expanding the norm $\norm{\mathbf{v}_{ij}}$, and rewriting $r_i$ using \eqref{rDef}, it can be shown that 
\begin{equation}\label{cosgammaVrRel}
    \cos{\gamma_i} = \frac{V_{r}}{V_{rel}}, \ V_{rel} = V\sqrt{2}\sqrt{1-\cos{\left(\phi_{Ri}-\phi_{Rj}\right)}}. 
\end{equation}
Now, if the two robots are on a collision course, then, the relative velocity $V_{r}<0$ and the relative speed $V_{rel}> 0$ as well. Thus, $\cos{\gamma_i}$ can be evaluated and the angle $\gamma_i$ lies in the domain $\frac{\pi}{2}<\gamma_i\leq \pi$, which is also the condition used in the definition of the repulsive potential field, \eqref{URepEq}.

\begin{remark}\label{TrigCondRemark}
The triggering condition defined in \eqref{URepEq} for a robot to ``activate'' the repulsive potential field can be evaluated by calculating the relative velocity $V_r$. This can be performed by the robots exchanging their state information - heading angles, positions, and speeds - with each other or by using on-board sensors to estimate these states, such as the infrared range and bearing sensor system, \cite{gutierrez2009open}. The experimental platform we use to demonstrate our algorithm provides each robot with the states of others. 
\end{remark}

Thus, a robot triggers the repulsive potential field only if it is on a collision path with another. In \eqref{cosgammaVrRel}, the relative speed condition $V_{rel}= 0$ holds when the heading angles of the two robots satisfy $\phi_{Ri}=\phi_{Rj}$. This condition occurs when they either are moving in parallel or are one behind another; in the latter case, since they are moving at the same linear speed, they also cannot overtake each other. In both these conditions, the repulsive field is not triggered. 

Prior to showing how the robot inputs are computed using the PF in \eqref{URepEq}, the gradients of $\mathcal{U}_{\text{Rep}i}$ are calculated. For robot $R_i$, these are
\begin{subequations}\label{GradURep}
\begin{align}
    \frac{\partial \mathcal{U}_{\text{Rep}i}}{\partial\left(x_{Rj}-x_{Ri}\right)} &= -\frac{\lambda V_{r}}{V_{rel}r^2}\left(2V_{\theta}\sin{\theta_i}+V_{r}\cos{\theta_i}\right) \label{GradURepx}\\
    \frac{\partial \mathcal{U}_{\text{Rep}i}}{\partial\left(y_{Rj}-y_{Ri}\right)} &= +\frac{\lambda V_{r}}{V_{rel}r^2}\left(2V_{\theta}\cos{\theta_i}-V_{r}\sin{\theta_i}\right). \label{GradURepy}
\end{align}
\end{subequations}

Now, the repulsive field assumes the vortex nature when the robot inputs are selected as
\begin{equation}\label{GradURepF}
    F_{xRi\text{Rep}} = -\frac{\partial \mathcal{U}_{\text{Rep}i}}{\partial\left(y_{Rj}-y_{Ri}\right)}, \ F_{yRi\text{Rep}} = +\frac{\partial \mathcal{U}_{\text{Rep}i}}{\partial\left(x_{Rj}-x_{Ri}\right)}.
\end{equation}
By fixing the signs of the components of the forces for each robot, it can be shown that the field ``rotates'' in the same direction for all robots; flipping the signs reverses the direction of rotation. As has been proved in \cite{Schey1997}, based on the expression of the repulsive field function and selecting the gradients according to \eqref{GradURepF}, the field becomes curl-free. 

\begin{figure}[!htpb]
    \centering
    \includegraphics[width=0.8\columnwidth]{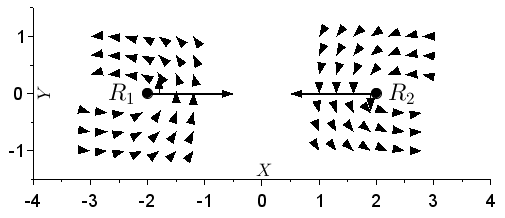}
	\caption{Gradients of the repulsive dynamic vortex PF when two robots are on a head-on collision course}
	\label{fig:DynVortexPF}
\end{figure}

The nature of the dynamic vortex PF can be understood from Fig.~\ref{fig:DynVortexPF}. In this figure, robots $R_1$ and $R_2$ are on a head-on collision course. The arrows around each robot show the direction of rotation of the other robot approaching it. As can be seen, since the signs of the gradients are the same for both robots, the direction of turn of both robots is the same. As the repulsive field is set to zero when the robots are moving away from each other, that is, when $V_r\geq0$, there are no arrows in the space ``behind'' each robot.

\begin{remark}\label{NonVortexRemark}
In the non-vortex case, the robot inputs are given by the negative gradients of the repulsive field and are of the form
\begin{equation}\label{GradURepFNonV}
    F_{xRi\text{Rep}} = -\frac{\partial \mathcal{U}_{\text{Rep}i}}{\partial\left(x_{Rj}-x_{Ri}\right)}, \ F_{yRi\text{Rep}} = - \frac{\partial \mathcal{U}_{\text{Rep}i}}{\partial\left(y_{Rj}-y_{Ri}\right)},
\end{equation}
which, as will be shown, may not guarantee collision avoidance for all cases of engagement.
\end{remark}

Various cases of robots, both cooperative and otherwise, interacting in the plane are considered next.

\subsection{Collision Avoidance between a Pair of Robots}\label{Sec:ClsAvd}

Collision avoidance is proved for the following cases: 
\begin{enumerate}[label=\textit{Case}~\arabic*.,leftmargin=*,align=left]
	\item A pair of cooperative robots on a collision path;
	\item A robot on a collision path with a stationary or non-cooperative robot;
	\item A cooperative robot and a non-cooperating robot actively seeking to collide by applying inputs derived using the Attractive PF.
\end{enumerate}
For \textit{Case}~3, the conditions on initial separation, initial relative velocity, and the PF parameters that ensure collision avoidance are identified.

For each of these cases, the types of robots that are considered are first defined and next, the proof of collision avoidance is presented.

\subsubsection{\textit{Case}~1}\label{Case1Sec}

\begin{definition}\label{CoopRobDef}
Cooperative robots are those that apply the same inputs as others when they are on a collision path.
\end{definition}
With this definition, collision avoidance for the case of cooperative robots is proved in the following theorem.

\begin{theorem}\label{CoopRobClsAvdThm}
Consider two robots moving in the plane. Let each robot know the position and velocity of the other, when measured in a common reference frame. If they are on a collision path and each applies inputs according to \eqref{GradURepF}, then, they avoid a collision with each other.
\end{theorem} 

\begin{proof}
The proof of this theorem rests on showing that the conditions for collision, $V_{r}<0$ and $V_{\theta}=0$, are \textit{never} satisfied if each robot applies inputs according to \eqref{GradURepF}. This result is proved by analysing the non-linear engagement dynamics, \eqref{VrVthDotAtt}, with the substitution of the robots' inputs, given by \eqref{GradURep} and \eqref{GradURepF}. Since, for robot $R_i$, these inputs are a function of the LOS angle $\theta_i$ with respect to itself, in the analysis, the inputs for both robots are expressed using a single LOS angle variable. This is possible, as, from Fig.~\ref{fig:2RobClsSchematic}, the two LOS angles satisfy the relation $\theta_j=\pi+\theta_i$. Based on this relation, the robot inputs satisfy $F_{xRi\text{Rep}}=-F_{xRj\text{Rep}}$ and $F_{yRi\text{Rep}}=-F_{yRj\text{Rep}}$.

Using these results and after much simplification, in the closed-loop, the dynamics of $V_r$ and $V_{\theta}$ become
\begin{subequations}\label{VrVthDotExpCL}
	\begin{align}
		\dot{V}_{r} &= \frac{V_{\theta}^2}{r} + 4\left(\frac{\lambda}{V_{rel}r^2}\right)V_{r}V_{\theta} \label{VriDotExpCL} \\
		\dot{V}_{\theta} &=  \frac{-V_{r}V_{\theta}}{r} +2\left(\frac{\lambda}{V_{rel}r^2} \right)V_{r}^2.\label{VthiDotExpCL}
	\end{align}
\end{subequations}
As can be observed, the equilibrium point of these dynamics is $V_r=V_{\theta}=0$.

Based on these dynamics, the occurrence of collision is demonstrated by analysing the Lyapunov function
\begin{equation}\label{LyapFnRep2CO}
    \mathcal{V}_{LRep}= r + \frac{1}{2}V_{\theta}^2 + \frac{1}{2}V_r^2.
\end{equation}
Similar to \eqref{VrVthCirc}, from \eqref{VrVthExp}, it can be shown that
\begin{equation}\label{VrVthVrelCirc}
    V_r^2+V_{\theta}^2=V_{rel}^2.
\end{equation}
This constraint also indicates that $\mathcal{V}_{LRep}= r + 0.5V_{rel}^2$. Thus, the Lyapunov function $\mathcal{V}_{LRep}$ can be interpreted as the sum of a potential energy component, measured by $r$, and a kinetic energy component, given by the relative speed $V_{rel}$. Note that if the robots collide, which implies $r=0$, and continue to move together, then $V_{rel}=V_r=V_{\theta}=0$ as well. Thus, $\mathcal{V}_{LRep}=0$ at the equilibrium point, which corresponds to collision. 

Since the robots are on a collision course, indicated by $V_r(t=0)=-2V<0$ and $V_{\theta}(t=0)=0$, the dynamics of $\mathcal{V}_{LRep}$ is analysed under these conditions. Let, at $t=0$, the robots be separated by distance $r_0$. Next, evaluating the derivative of $\mathcal{V}_{LRep}$ along the trajectories of the states $\dot{r}=V_r$ and $V_r,V_{\theta}$ as given by \eqref{VrVthDotExpCL}, results in the expression
\begin{equation}\label{LyapFnRep2CODot1}
	\dot{\mathcal{V}}_{LRep} = V_r\left(1 + \frac{6\lambda}{V_{rel}r^2}V_{r}V_{\theta}\right).
\end{equation}
Thus, if $\dot{\mathcal{V}}_{LRep}<0$ for all values of the states, then, the equilibrium point is stable, which implies that collision occurs. However, the converse also holds true, an unstable equilibrium point, indicated by $\dot{\mathcal{V}}_{LRep}>0$, results only with an increase in the separation distance $r$, implying that the robots do not collide.

From \eqref{LyapFnRep2CODot1}, consider the ratio
\begin{equation}\label{VLVrComp}
    \frac{\dot{\mathcal{V}}_{LRep}}{V_r}=\left(1 + \frac{6\lambda}{V_{rel}r^2}V_{r}V_{\theta}\right).
\end{equation}
Since, at $t=0$, $\dot{V}_{\theta}>0$ and $V_r<0$, the robots begin to turn away from each other, while simultaneously also approaching each other, that is, the separation distance, $r$, between the robots reduces. Thus, for some $\lambda$ and $r_0$, a time instant, $t_1>0$, can be found when the ratio $0<\frac{\dot{\mathcal{V}}_{LRep}}{V_r}<1 \ \forall \ 0\leq t<t_1$. This implies that during this interval, the separation distance \emph{cannot} reduce faster than $\mathcal{V}_{LRep}$. Thus, even if $t=t_1$ is the time instant when $\mathcal{V}_{LRep}$ reaches its minimum, that is, when $\dot{\mathcal{V}}_{LRep}=0$, since $V_r\neq0$ at $t=t_1$, there exist \emph{non-zero} values of $r,V_r,V_{\theta}$ that satisfy $V_{r}V_{\theta}=-\frac{V_{rel}r^2}{6\lambda}<0$. Thus, there exists a non-zero separation distance when the Lyapunov function $\mathcal{V}_{LRep}$ reaches its minimum. This also implies that for $t>t_1$, the derivative of $\mathcal{V}_{LRep}$ changes sign.

Now, for $t>t_1$, since $\dot{\mathcal{V}}_{LRep}>0$, the origin becomes an unstable equilibrium point. Moreover, as the separation distance increases, the terms $\frac{V_{\theta}^2}{r}>0$ and $\frac{-V_{r}V_{\theta}}{r}>0$ dominate the derivatives of $V_r$ and $V_{\theta}$, respectively, implying that, there exists a time instant, say $t=t_2>t_1$, when $V_r=0$ and $V_{\theta}=V_{rel}$. Thus, as the trajectories in the $V_r-V_{\theta}$ space are attracted to this point, the conditions for collision do not hold any longer and the robots now ``deactivate'' the vortex repulsive field. 

\end{proof}

The following remarks help in obtaining an intuitive understanding of the effect of the inputs designed using the vortex repulsive field and the proof of Theorem~\ref{CoopRobClsAvdThm}:
\begin{enumerate}[wide=0pt,label=\textit{\roman*}.]
    \item The inputs to the robots are
        \begin{subequations}\label{URepF2}
            \begin{align}
                F_{xRi\text{Rep}} &= -\frac{\lambda V_{r}}{V_{rel}r^2}\left(2V_{\theta}\cos{\theta_i}-V_{r}\sin{\theta_i}\right)  \label{FRepx2}\\
                F_{yRi\text{Rep}} &= -\frac{\lambda V_{r}}{V_{rel}r^2}\left(2V_{\theta}\sin{\theta_i}+V_{r}\cos{\theta_i}\right), \label{FRepy2}
            \end{align}
        \end{subequations}
        from which can be seen that the magnitude of these forces increases with decreasing $r$. Thus, if the two robots are initially located close to each other, $r\approx0$, then, for any $\lambda>0$, these forces ensure that the robots turn away from each other at a faster rate. Similarly, if the robots are far away, $r\gg0$, then, these forces do not significantly influence the robots' maneuvers. The effect of bounded robot inputs on collision avoidance is discussed in a later section.
    \item As the robots deactivate the repulsive field once $V_r=0$ and $V_{\theta}=V_{rel}$ at $t=t_2$, the dynamics of the relative accelerations reduce to 
        \begin{equation*}
            \dot{V}_{r} = \frac{V_{\theta}^2}{r}>0, \ \dot{V}_{\theta} = 0,
        \end{equation*}
    implying the trajectories in the $V_r-V_{\theta}$ space do not enter the domain $V_r<0$. Thus, the forces eliminate oscillatory behaviour during collision avoidance. 
    \item Also, at $t=t_2$, the behaviour of the robots will now be defined by the attractive potential field, which, as discussed in Sec.~\ref{Sec:AttPF}, ensures that the robots reach their goal locations.
\end{enumerate}

\subsubsection{\textit{Case}~2}\label{Case2Sec}

In this section, it is shown that the inputs \eqref{GradURepF} also aid in a robot avoiding collisions with stationary or non-cooperative robots. Prior to presenting this proof, a non-cooperative robot is formally defined. 

\begin{definition}\label{NonCoopRobDef}
A non-cooperative robot is one that moves at a constant velocity and does not actively avoid collisions with other robots.
\end{definition}

\begin{theorem}\label{NonCoopRobClsAvdThm}
A robot applying the inputs \eqref{GradURepF} can avoid a collision with both stationary and non-cooperative robots.
\end{theorem}

\begin{proof}
Consider a robot, $R_i$, which is on a collision path with either a stationary or non-cooperative robot. Since, both stationary and non-cooperative robots do not apply the inputs \eqref{GradURepF}, in the dynamics \eqref{VrVthDotExp} for robot $R_i$, the terms $F_{xRj}=F_{yRj}=0$. The relative velocity expressions for a stationary robot reduce to $V_{r}=-V\cos{\left(\phi_{Ri}-\theta_{i}\right)}$ and $V_{\theta}=- V\sin{\left(\phi_{Ri}-\theta_{i}\right)}$, while, for a constant velocity non-cooperative robot, the relative velocity terms are as given in \eqref{VrVthExp}. However, in both cases, the relative accelerations have the same expressions.

Thus, when robot $R_i$ applies the inputs \eqref{GradURepF}, in the closed-loop, the relative accelerations become
\begin{subequations}\label{VrVthDotExpCL2}
	\begin{align}
		\dot{V}_{r} &= \frac{V_{\theta}^2}{r} + 2\left(\frac{\lambda}{V_{rel}r^2}\right)V_{r}V_{\theta} \label{VriDotExpCL2} \\
		\dot{V}_{\theta} &=  \frac{-V_{r}V_{\theta}}{r} +\left(\frac{\lambda}{V_{rel}r^2} \right)V_{r}^2.\label{VthiDotExpCL2}
	\end{align}
\end{subequations}
Now, by defining a Lyapunov function, denoted by $\mathcal{V}_{LRepNC}$ with an expression identical to the one in Theorem~\ref{CoopRobClsAvdThm}, its derivative is analysed. The derivative of $\mathcal{V}_{LRepNC}$ is given by
\begin{align}\label{LyapFnDotCONonCO}
	\dot{\mathcal{V}}_{LRepNC} = -|V_r|\left(1 + \frac{3\lambda}{V_{rel}r^2}V_{r}V_{\theta}\right).
\end{align}
In this case, employing arguments similar to those used in the proof of Theorem~\ref{CoopRobClsAvdThm}, it can be shown that, the cooperative robot is able to avoid a collision with the non-cooperative one. 
\end{proof}

\begin{remark}\label{MagTermsRemark}
Note the difference in the coefficients in the terms in \eqref{VrVthDotExpCL} and \eqref{VrVthDotExpCL2} derived using the repulsive field inputs, \eqref{GradURepF}. This difference is natural, since in \textit{Case}~1, the robots reciprocate and turn away from each other, thus increasing the magnitude of the relative acceleration component. On the other hand, in \textbf{Case}~2, since it is only one of the robots that turns away from the non-cooperative robot, the relative acceleration is smaller. 
\end{remark}

\begin{remark}\label{DiffVRemark}
Since the proof of collision avoidance is shown in the relative velocity space, our algorithm is applicable even when cooperative robots on a collision path move at different constant linear speeds. 
\end{remark}

\subsubsection{\textit{Case}~3}\label{Case3Sec}

In this case, a robot, say $R_j$, that is actively seeking to collide with a cooperative robot, $R_i$, is considered. An ``attacking'' robot is defined as

\begin{definition}\label{AttackRobDef}
An attacking robot is one that chooses a cooperative robot as a goal (or target) and applies its inputs according to the inputs defined by the negative gradient of the attractive PF, in \eqref{FxyAtt}.
\end{definition}

Prior to demonstrating collision avoidance against an attacking robot, a few notations are introduced based on the engagement geometry shown in Fig.~\ref{fig:2RobClsSchematic}. Let $R_2$ be the attacking robot which is trying to collide with the cooperative robot $R_1$. Thus, the attractive PF for $R_2$ is
\begin{equation}\label{UAttackPF}
    \mathcal{U}_{\text{Attack}2} = \kappa\sqrt{\left(x_{R2}-x_{R1}\right)^2 + \left(y_{R2}-y_{R1}\right)^2},
\end{equation}
and following the discussion presented in Sec.~\ref{Sec:AttPF}, its inputs are chosen as the negative gradients
\begin{subequations}\label{FxyAttack}
\begin{align}
    F_{x2} &= \frac{-\partial \mathcal{U}_{\text{Attack}2}}{\partial\left(x_{R2}-x_{R1}\right)}=+\kappa \cos{\left(\theta_2\right)}=-\kappa \cos{\left(\theta_1\right)} \label{FxAttack}\\
    F_{y2} &= \frac{-\partial \mathcal{U}_{\text{Attack}2}}{\partial\left(y_{R2}-y_{R1}\right)}=+\kappa \sin{\left(\theta_2\right)}=-\kappa \sin{\left(\theta_1\right)}, \label{FyAttack}
\end{align}
\end{subequations}
since, $\theta_2=\pi+\theta_1$.

If $R_1$ and $R_2$ are on a collision path, then, the cooperative robot $R_1$ considers the robot $R_2$ as a dynamic obstacle. Now, $R_1$ ``projects'' the dynamic vortex PF around robot $R_2$ and applies its inputs in order to avoid a collision; these inputs are given by \eqref{GradURepF}. This projection of the repulsive field by one robot on another is visualised in Fig.~\ref{fig:DynVortexPF}. The dynamics of the relative velocities become
\begin{subequations}\label{VrVthDotExpCL3}
	\begin{align}
		\dot{V}_{r} &= \frac{V_{\theta}^2}{r} + 2\left(\frac{\lambda}{V_{rel}r^2}\right)V_{r}V_{\theta} - \kappa \label{VriDotExpCL3} \\
		\dot{V}_{\theta} &=  \frac{-V_{r}V_{\theta}}{r} +\left(\frac{\lambda}{V_{rel}r^2} \right)V_{r}^2.\label{VthiDotExpCL3}
	\end{align}
\end{subequations}
As can be observed, if the robot $R_1$ does not project a repulsive field, that is, if $\lambda=0$, and if the robots are moving such that the collision conditions are satisfied, then, $R_2$ will collide with $R_1$. This is also supported by the discussion in Sec.~\ref{Sec:AttPF}.

The conditions when collision is avoided are identified in the following theorem.

\begin{theorem}\label{AttackRobClsAvdThm}
Given the PF parameters $\kappa,\lambda>0$, if the initial separation distance, $r_0$, between the attacking and cooperative robots satisfies $r_0 \geq\sqrt{3\lambda V}$, where $V$ is the linear speed of each robot, then, the cooperative robot is able to avoid colliding with the attacking robot.
\end{theorem}

\begin{proof}
Based on the dynamics \eqref{VrVthDotExpCL3}, collision avoidance is demonstrated by analysing the Lyapunov function
\begin{equation}\label{LyapFnAtck}
    \mathcal{V}_{LAtck}= \kappa r + \frac{1}{2}V_{\theta}^2 + \frac{1}{2}V_r^2,
\end{equation}
with equilibrium conditions $r=V_r=V_{\theta}=0$; these equilibrium conditions are similar to those selected in \textbf{Case}~1. From \eqref{VrVthDotExpCL3} as well as the relation $\dot{r}=V_r$, the derivative of $\mathcal{V}_{LAtck}$ becomes
\begin{equation}\label{LyapFnAtckDot}
    \dot{\mathcal{V}}_{LAtck} = \frac{3\lambda}{V_{rel}r^2}V_{r}^2V_{\theta}.
\end{equation}

First, the case when the robots are initially on a head-on collision course is considered. Let $V_r(t=0)=-2V$ and $V_{\theta}(t=0)=0$, then, there exists a time interval, $0<t\leq t_1$, when $\dot{\mathcal{V}}_{LAtck}>0$, since, during this interval, from \eqref{VthiDotExpCL3}, the relative acceleration $\dot{V}_{\theta}>0$. Thus, in this interval, the equilibrium point is unstable, indicating that collision is avoided. 

Since the states $V_r,V_{\theta}$ satisfy the condition \eqref{VrVthVrelCirc}, let, at $t=t_1$, $V_{\theta}(t=t_1)=2V$, $V_r=0$, and $r(t=t_1)=r_1>0$. Now, the cooperative robot $R_1$ stops projecting the vortex repulsive field, since the triggering condition given in \eqref{URepEq} does not hold. Thus, at $t=t_1$, the dynamics \eqref{VrVthDotExpCL3} become
\begin{equation}\label{VrVthDotExpCL3t1}
    \dot{V}_{r} = \frac{V_{\theta}^2}{r} - \kappa, \ \dot{V}_{\theta} = \frac{-V_{r}V_{\theta}}{r}.
\end{equation}
Now, if for some $t=t_2\geq t_1$, the separation distance $r$ and the relative velocity $V_{\theta}$ are such that $\dot{V}_{r}<0$ then, the attacking robot begins to approach the cooperative robot. However, the cooperative robot again projects the repulsive field and the collision avoidance maneuver starts again. Since, for $t\geq t_2$, $V_{\theta}>0$, the derivative of the Lyapunov function again satisfies $\dot{\mathcal{V}}_{LAtck}>0$, thus, indicating an unstable equilibrium $(r,V_r,V_{\theta})=0$. 

This result holds for any $\lambda,\kappa>0$ and $r_0$.

Suppose the initial conditions are $V_r(t=0)<0$ and $V_{\theta}(t=0)<0$. Now, $\dot{\mathcal{V}}_{LAtck}<0$, indicating a possibility of the equilibrium states being stable. However, by imposing conditions on the parameter $\lambda$ and the initial separation distance $r_0$, it may be possible that the robots have a non-zero separation distance at the time instant when $\mathcal{V}_{LAtck}=0$. As used in the proof of Theorem~\ref{CoopRobClsAvdThm}, compute the ratio
\begin{align}\label{VAtckVrRatio}
    \frac{\dot{\mathcal{V}}_{LAtck}}{\dot{r}} &= \frac{3\lambda}{V_{rel}r^2}V_{r}V_{\theta} \leq \frac{3\lambda}{2r_0^2}V_{rel}. \nonumber
    \intertext{Thus, for a chosen $\lambda>0$, if}
    r_0 &\geq\sqrt{3\lambda V} \Rightarrow \frac{\dot{\mathcal{V}}_{LAtck}}{\dot{r}} < 1,
\end{align}
indicating that $r$ cannot go to zero before $\mathcal{V}_{LAtck}$ reaches a minimum, implying that there will be a non-zero separation distance between the attacking and cooperative robots.

Note that $\mathcal{V}_{LAtck}$ reaches a minimum when $V_{\theta}=0$, since, for $V_r(t=0)<0$ and $V_{\theta}(t=0)<0$, the relative acceleration $\dot{V}_{\theta}>0$, from \eqref{VthiDotExpCL3}. Denote this time instant as $t=t_3$. Now, if $V_r(t=t_3)<0$, it has already been shown that collision does not occur. On the other hand, if $V_r(t=t_3)\equiv0$, depending on the magnitudes of the parameters, the robots are still separated by a non-zero distance, but are moving parallel to each other.

This analysis shows that the trajectories in the $V_r-V_{\theta}$ space oscillate around $V_r=0$ and $V_{\theta}>0$, indicating that the cooperative robot keeps turning away from the attacking robot until the collision avoidance conditions are violated. Once this happens and the attacking robot maneuvers to collide with the cooperative one, the latter again begins to turn away to avoid colliding with the former.
\end{proof}

\subsection{Collision avoidance for non-point cooperative robots with bounded inputs}\label{Sec:NonZeroRob}

The results presented so far consider the robots to be points. In this section, the analysis on collision avoidance is extended to cooperative robots that are defined by circles of radius $R_{\text{Rob}}>0$ and also have limits on their accelerations. This analysis will aid in determining:
\begin{enumerate}[wide=0pt,label=\textit{\roman*}.]
    \item the influence of limits that naturally bound the acceleration inputs that can be applied by a robot, when the input commanded by the repulsive field is greater than this limit;
    \item the amount of space that is needed by a robot while performing the collision avoidance maneuver, both against other mobile robots as well as stationary obstacles; and
    \item the role of the PF parameters $\lambda$ and $\kappa$ in the case of avoiding collision with an attacking robot.
\end{enumerate}

First, the relation between the acceleration inputs calculated using the vortex potential field and limits on these inputs is derived. From \eqref{GradURep} and \eqref{GradURepF}, for a cooperative robot $R_i$, it can be seen that if the separation distance $r$ is very small, then clearly the inputs will have large magnitudes. As discussed in Sec.~\ref{Case1Sec}, it is this relation that enables point robots to turn away and avoid a collision even if their initial separation distance is very small. However, if these inputs are bounded, say according to $|F_{xRi\text{Rep}},F_{yRi\text{Rep}}|\leq F_{\lim}$, then, for small values of $r$, the inputs become saturated at $F_{\lim}$. 

In such cases, similar to \eqref{VAtckVrRatio}, for a chosen repulsive field parameter $\lambda$ and by setting $V_{rel}=2V$, by introducing another design parameter $r^*>0$, the robots' inputs can be defined as the output of the function
\begin{subequations}\label{FxySatFn}
\begin{align}
     F_{xRi\text{Rep}} &= \begin{cases}
                            -\dfrac{\lambda }{V_{rel}r^2}\left(2V_{r}V_{\theta}\cos{\theta_i}-V_{r}^2\sin{\theta_i}\right) \ & \text{if } r>r^* \\
                            -F_{\lim}\text{sign}\left(2V_{r}V_{\theta}\cos{\theta_i}-V_{r}^2\sin{\theta_i}\right) & \ \text{otherwise}
                            \end{cases} \label{FxSatFn}\\
     F_{yRi\text{Rep}} &= \begin{cases}
                            -\dfrac{\lambda }{V_{rel}r^2}\left(2V_rV_{\theta}\sin{\theta_i}+V_{r}^2\cos{\theta_i}\right) \ & \text{if } r>r^* \\
                            -F_{\lim}\text{sign}\left(2V_rV_{\theta}\sin{\theta_i}+V_{r}^2\cos{\theta_i}\right) & \ \text{otherwise}
                            \end{cases}. \label{FySatFn}
\end{align}
\end{subequations}

The design parameter $r^*$ can be expressed in terms of a minimum initial separation distance and the acceleration limit $F_{\lim}$ that ensures that robots, each of radius $R_{\text{Rob}}$, can avoid ``grazing'' each other while performing the collision avoidance maneuver. To do so, consider two such cooperative robots, $R_1$ and $R_2$, as shown in Fig.~\ref{fig:2RobClsSchematicHO}. Let these robots be initially separated by a distance $2l\leq r^*$ and also be moving on a head-on collision course, so that the relative speed is the highest at $V_{rel}=2V$. As has been proved in Theorem~\ref{CoopRobClsAvdThm}, the robots begin to turn away from each other to avoid a collision. However, given their initial separation distance, the magnitudes of the inputs to each robot become $F_{\lim}$. As a result, the robots trace a circle with radius
\begin{equation}\label{RTurnExp}
   R_{\text{turn}} = \frac{V^2}{F_{\lim}}.
\end{equation}
This result is derived using the expression
\begin{equation*}
    R_{\text{turn}} = \Bigg|\frac{\left(\dot{x}^2+\dot{y}^2\right)^{\frac{3}{2}}}{\dot{x}\ddot{y}-\dot{y}\ddot{x}}\Bigg|,
\end{equation*}
where, $\dot{x},\dot{y}$ and $\ddot{x},\ddot{y}$ are the velocity and acceleration components in Cartesian space and are given by \eqref{RobotDyn} and \eqref{RobotDynAccl}, respectively.

Since the robots are homogeneous, the value $F_{\lim}$ is the same for each robot. Owing to this homogeneity, the motions of the two robots are symmetric. Now, the point when the separation distance is the smallest, denoted by the distance $2d$ in Fig.~\ref{fig:2RobClsSchematicHO}, can be shown to satisfy
\begin{equation}\label{dSepExp}
   d = \sqrt{R_{\text{turn}}^2+l^2}-R_{\text{turn}}.
\end{equation}
This occurs at the time-instant defined by $t=t_1$ in the proof of Theorem~\ref{CoopRobClsAvdThm}. Thus, the circles traced by the robots touch each other if the initial separation distance $l=0$, which implies collision.

For the robots to avoid grazing each other, the distance $d$ should satisfy $2d\geq2R_{\text{Rob}}$, which, in turn, leads to
\begin{equation}\label{lFlimExp}
   R_{\text{turn}} \leq \frac{l^2-R_{\text{Rob}}^2}{2R_{\text{Rob}}}\ \Rightarrow F_{\lim} \geq \frac{2R_{\text{Rob}}V^2}{l^2-R_{\text{Rob}}^2}.
\end{equation}
This result clearly indicates that if the initial difference $\left(l^2-R_{\text{Rob}}^2\right)$ is small, the robot radius $R_{\text{Rob}}$ and speed $V$ are large, then, the acceleration limit has to be proportionately large for the robots to turn away from each other and avoid a collision. Note that if $F_{\lim}$ is large, then, the radius $R_{\text{turn}}$ is also small. 

In the case of a non-cooperative robot that is moving at a constant velocity, the acceleration limit should satisfy
\begin{equation}\label{lFlimExpNonCoop}
   F_{\lim} \geq \frac{4R_{\text{Rob}}V^2}{l^2-4R_{\text{Rob}}^2},
\end{equation}
which is \textit{larger} than what is required when two cooperative robots are avoiding each other. This too should be expected since it is only the cooperative robot that is performing the collision avoidance maneuver.

Thus, based on the acceleration capacity of the robot, their initial separation distance and hence, $r^*$, can be suitably selected.

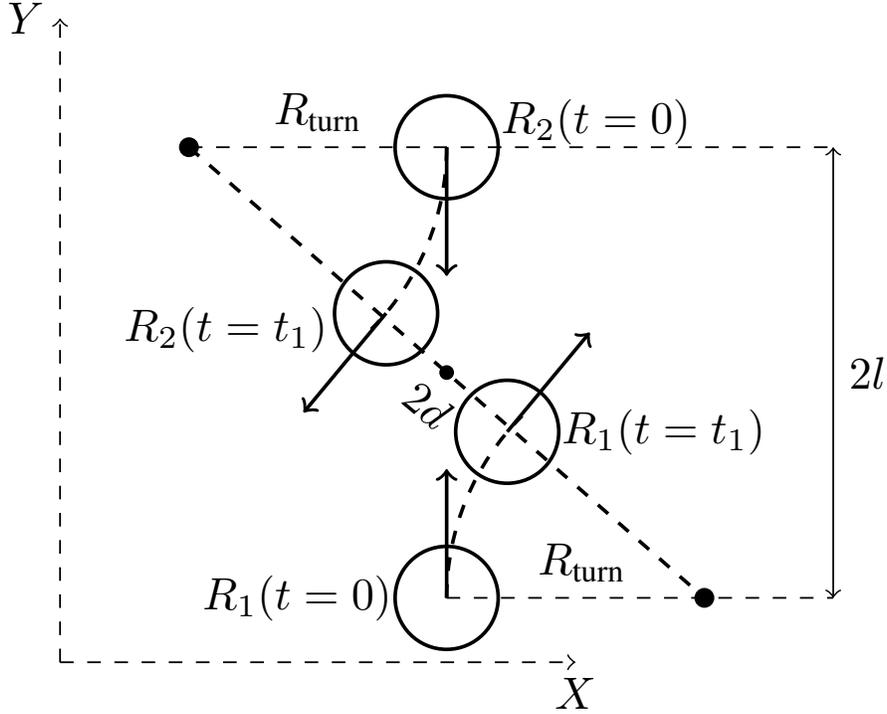
\begin{figure}[ht]
\centering
\resizebox{0.75\columnwidth}{!}{
	
	\begin{tikzpicture} 
		
			\draw[->,dashed] (-2,0) -- (2,0) node[pos=1,below] {$X$}; 
			\draw[->,dashed] (-2,0) -- (-2,5) node[pos=1,left] {$Y$};
			
			\draw[thick] (1,0.5) circle (0.4cm);
			\draw (0.7,0.5) node[left] {$R_1(t=0)$};
			\draw[->,thick] (1,0.5) -- (1,1.5); 
			
			\draw[thick] (1,4) circle (0.4cm);
			\draw (1.3,4.2) node[right] {$R_2(t=0)$};
			\draw[->,thick] (1,4) -- (1,3); 
			
			\draw[dashed,thick] (1,0.5) arc (180:135:2cm); 
			\draw[dashed,thick] (1,4) arc (360:315:2cm); 
			
			\draw[thick] (1.47,1.79) circle (0.4cm);
			\draw (1.77,1.79) node[right] {$R_1(t=t_1)$};
			\draw[->,thick] (1.47,1.79) -- (2.11,2.56); 
			
			\draw[thick] (0.53,2.71) circle (0.4cm);
			\draw (0.2,2.58) node[left] {$R_2(t=t_1)$};
			\draw[->,thick] (0.53,2.71) -- (-0.11,1.94); 
			
			
			\draw[fill=black] (-1,4) circle (0.07cm);
			\draw[fill=black] (3,0.5) circle (0.07cm);
			
			\draw[dashed] (-1,4) -- (4,4) node[pos=0.2,above] {$R_{\text{turn}}$};
			\draw[dashed] (1,0.5) -- (4,0.5) node[pos=0.35,above] {$R_{\text{turn}}$};
			
			\draw[dashed,thick] (-1,4) -- (3,0.5);
			
			\draw[<->] (4,4) -- (4,0.5) node[pos=0.5,right] {$2l$};
			
			\node[rotate=-40] (M) at (0.85, 2) {$2d$};
			\draw[fill=black] (1,2.25) circle (0.05cm);
			
	\end{tikzpicture}
	
	}
    \caption{A pair of cooperative robots defined by circles avoiding colliding with each other} 
	\label{fig:2RobClsSchematicHO}
\end{figure}

\begin{remark}\label{FreeSpaceRemark}
\textbf{Free Space Needed}: The result \eqref{lFlimExp} and \eqref{RTurnExp} can be used to determine the space that is needed by a cooperative robot to avoid collisions with other robots/obstacles in its vicinity. Since a robot turns away from obstacles at its limit $F_{\lim}$, then it traces a circle of radius $R_{\text{turn}}$; this can happen if the obstacles are at distances that satisfy \eqref{lFlimExp}. Consider a scenario where a cooperative robot is at the open side of an enclosure closed on 3-sides. In this scenario, since the direction of turn of the robot is fixed, the free space that will be needed by the robot to avoid colliding with any edge is the area traced by a semi-circle of radius $\left(R_{\text{turn}}+R_{\text{Rob}}\right)$.
\end{remark}

\subsection{Collision Avoidance between Multiple Robots}\label{Sec:MultRob}

Proof of collision avoidance using the vortex repulsive field can be directly extended when multiple robots are involved. Collision amongst multiple robots can occur when the condition to activate the repulsive field, \eqref{URepEq}, is true for several pairs of robots. Now, the closed-loop dynamics of pairwise relative velocities are examined and the proof of Theorem~\ref{CoopRobClsAvdThm} is applied to show that for each pair, collision is avoided. Note that when $N>2$ robots are involved, the number of pairwise relative velocities, $V_{r}$ and $V_{\theta}$, are $N(N-1)/2$, each.

For simplicity, let the robots be defined as points. 
Now, the inputs for collision avoidance for each robot $R_i, \ i=1,\cdots,N$, are given by
\begin{subequations}\label{SumFxyRep}
\begin{align}
    F_{xRi\text{Rep}} &= -\lambda\sum\limits_{j|\cos{\gamma_{ij}}<0}{\frac{ V_{rij}}{V_{relij}r_{ij}^2}\left(2V_{\theta ij}\cos{\theta_{ij}}-V_{rij}\sin{\theta_{ij}}\right)} \label{FxRepSum}\\
    F_{yRi\text{Rep}} &= -\lambda\sum\limits_{j|\cos{\gamma_{ij}}<0}{\frac{ V_{rij}}{V_{relij}r_{ij}^2}\left(2V_{\theta ij}\sin{\theta_{ij}}+V_{rij}\cos{\theta_{ij}}\right)}, \label{FyRepSum}
\end{align}
\end{subequations}
where, $j=1\cdots N, \ i\neq j$; for the pair of robots $R_i$ and $R_j$, $r_{ij}$ denotes the separation distance between them; $V_{rij}$ and $V_{\theta ij}$ are the relative velocities; $\theta_{ij}$ is the LOS angle; and $V_{relij}$ is the relative speed, given by \eqref{cosgammaVrRel}. The following identities hold for a pair of robots $R_i$ and $R_j$, for which the triggering condition, $\cos{\gamma_{ij}}<0$, is true:
\begin{align}\label{MultrobRels}
     F_{xRij\text{Rep}} &= -F_{xRij\text{Rep}}, \ F_{yRij\text{Rep}} = -F_{yRij\text{Rep}}, \nonumber \\
     \theta_{ji} &= \pi + \theta_{ij}, \ V_{rij} = V_{rji}, \ V_{\theta ij} = V_{\theta ji}.
\end{align}

In the same spirit as Theorem~\ref{CoopRobClsAvdThm}, collision avoidance between multiple point robots can be proved by defining the Lyapunov function
\begin{equation}\label{LyapMultRob}
    \mathcal{V}_{LRepMult}= \sum\limits_{i,j|\cos{\gamma_{ij}}<0}{\left(r_{ij} + \frac{1}{2}\left(V_{\theta ij}^2 + V_{r ij}^2\right)\right)}
\end{equation}
and evaluating its derivatives along the trajectories of the relative accelerations and that of the separation distance, given by $\dot{r}_{ij}=V_{rij}$. Using \eqref{MultrobRels} and after simplification, it can be shown that
\begin{equation}\label{LyapFnRepDotMR}
    \dot{\mathcal{V}}_{LRepMult} = -\sum{|V_{rij}|\left(1 + \frac{3N'\lambda}{V_{relij}r_{ij}^2}V_{rij}V_{\theta ij}\right)},
\end{equation}
where, $N'$ is the number of cooperative robots that are performing the collision avoidance maneuvers. Now, using the same arguments as used in proving Theorem~\ref{CoopRobClsAvdThm}, it can be shown that the equilibrium conditions $r_{ij}=V_{rij}=V_{\theta ij}=0$ are unstable for any pair of robots. Thus, collision is avoided between any pair of robots that are on a collision course.

\textbf{Remarks}: For robots bounded by circles and with bounded inputs, the results of Sec.~\ref{Sec:NonZeroRob} hold in the multiple robots case as well. The coefficient $3N'$, in \eqref{LyapFnRepDotMR}, appears when a pair of cooperative robots are avoiding each other, $N'=2$ in \eqref{LyapFnRep2CODot}, as well as when a single cooperative robot is maneuvering to avoid colliding with a non-cooperative robot, $N'=1$ in \eqref{LyapFnDotCONonCO}.

\subsection{Implementation}

As suggested in \cite{DeLuca1994}, with the use of potential fields, a robot can avoid collisions and navigate to its goal if its heading angle, $\phi_{Ri}$, is controlled to achieve the desired heading angle, $\phi_{Ri\text{Des}}$, given by
\begin{equation}\label{phiDes}
    \phi_{Ri\text{Des}} = \tan^{-1}{\left(\frac{F_{yRi\text{Att}}+\sum\limits_{j}{F_{yRij\text{Rep}}}}{F_{xRi\text{Att}}+\sum\limits_{j}{F_{xRij\text{Rep}}}}\right)}.
\end{equation}
In the context of this paper, $F_{xRi\text{Att}}$ and $F_{yRi\text{Att}}$ are given by \eqref{FxyAtt}, while $F_{xRi\text{Rep}}$ and $F_{yRi\text{Rep}}$ are calculated using \eqref{GradURep} and \eqref{GradURepF}.

To ensure $\phi_{Ri}\to\phi_{Ri\text{Des}}$, the angular speed, $\omega_{Ri}$ in \eqref{phiDyn}, can be designed as the output of the controller
\begin{equation}\label{omegaRiCL}
    \omega_{Ri}=K_p\left(\phi_{Ri\text{Des}}-\phi_{Ri}\right),
\end{equation}
where, $K_p>0$ is the control gain. This control law can also be chosen as a Proportional-Integral controller so that bounded disturbances acting on the robot can be suppressed.

In this section, it is shown that the implementation, \eqref{phiDes} and \eqref{omegaRiCL}, leads to collision avoidance between robots as well as navigation towards the goal location. For ease of understanding, these cases are considered separately. First, the case of navigating towards the goal is considered, that is, in \eqref{phiDes}, the repulsive field components are set to zero. Thus, as discussed in Sec.~\ref{Sec:AttPF}, from the expressions for $F_{yRi\text{Att}}$ and $F_{xRi\text{Att}}$, as given by \eqref{FxyAtt}, the desired heading angle for robot $R_i$ becomes
\begin{align}\label{phiDesAttExp}
    \phi_{Ri\text{Des}} &= \tan^{-1}{\left(\frac{\kappa \sin{\left(\theta_i\right)}}{\kappa \cos{\left(\theta_i\right)}}\right)} \nonumber
    \intertext{implying}
    \phi_{Ri\text{Des}} &= \theta_i \ \text{or} \ \phi_{Ri\text{Des}} = \theta_i + \pi.
\end{align}
Since it has been shown in Sec.~\ref{Sec:AttPF} that the relative velocity states $V_{\theta}=0$ and $V_{r}<0$ are stable equilibrium states, the valid solution for the heading angle is $\phi_{Ri\text{Des}} = \theta_i$. 

Now, consider the case of collision avoidance between a pair of cooperative robots $R_i$ and $R_j$. using the vortex repulsive field. By setting $F_{yRi\text{Att}}=0$ and $F_{xRi\text{Att}}=0$ in \eqref{phiDes} and from \eqref{GradURep} and \eqref{GradURepF}, the desired heading angle for $R_i$ becomes
\begin{equation}\label{phiDesRepExp}
    \phi_{Ri\text{Des}} = \tan^{-1}{\left(\frac{2V_{\theta}\sin{\theta_i}+V_{r}\cos{\theta_i}}{2V_{\theta}\cos{\theta_i}-V_{r}\sin{\theta_i}}\right)}.
\end{equation}
The desired heading angle for $R_j$ can be similarly derived. Using the result $\theta_j=\theta_i+\pi$, it can be shown that
\begin{subequations}\label{PhiDesijRepExp2}
    \begin{align}
        \tan{\phi_{Ri\text{Des}}} &= \tan{\phi_{Rj\text{Des}}}, \label{PhiDesjRepExpf}
        \intertext{implying,}
        \phi_{Ri\text{Des}} &= \phi_{Rj\text{Des}} \ \text{or}  \label{PhiDesjRepExpf1} \\
        \phi_{Ri\text{Des}} &= \pi + \phi_{Rj\text{Des}}. \label{PhiDesjRepExpf2}
    \end{align}
\end{subequations}
Since the repulsive field is triggered only when robots are on a collision path, their heading angles cannot be the same, thus the solution \eqref{PhiDesjRepExpf1} does not hold. As a result, the heading angles of the two cooperative robots satisfy $\phi_{Ri\text{Des}} = \pi + \phi_{Rj\text{Des}}$, implying that while they are on a collision course, they turn away from each other. This result also matches with that derived in Sec.~\ref{Sec:NonZeroRob}. 

When only robot $R_i$ is cooperative and robot $R_j$ is either non-cooperative or attacking, a similar analysis can be performed to find a relation between $\phi_{Ri\text{Des}}$ and the heading angle of robot $R_j$. When $R_j$ is non-cooperative, $\phi_{Rj\text{Des}}=\phi_{Rj} \ \forall \ t\geq0$, since $R_j$ does not change its path. When $R_j$ is attacking, $\phi_{Rj\text{Des}}=\theta_{i}$. As has been proved in the earlier sections, collision is avoided between the cooperative robot and both non-cooperative or attacking types of robots.

Note that while calculating $\phi_{Ri\text{Des}}$, either in \eqref{phiDesRepExp} or in \eqref{phiDesAttExp}, the PF parameters $\kappa$ and $\lambda$ do not appear, indicating that these parameters do not influence the desired heading angle. However, from the closed-loop dynamics of the relative velocities $V_r$ and $V_{\theta}$, it can be seen that the PF parameters affect the magnitudes of these relative velocities, which in turn, determine $\phi_{Ri\text{Des}}$.

These theoretical results are supported by experimental, which are presented next.

\section{Experimental Results}\label{Sec:ExpRes}

The dynamic vortex PF algorithm is implemented on the QBOT 2E mobile robot platform by Quanser\footnote{\url{https://www.quanser.com/products/qbot-2e/}}. All experiments are conducted in the Autonomous Vehicles Research Studio\footnote{\url{https://www.quanser.com/products/autonomous-vehicles-research-studio/}}, also from Quanser. Some of the key features of the robot and the Studio are listed in Table~\ref{tab:Qbot2EParams}.

\begin{table}[htpb!]  
	\centering
	\caption{Key parameters of the experimental platform}
  \label{tab:Qbot2EParams}
    \begin{tabular}{c|c}
			\toprule
			\multicolumn{2}{c}{QBOT 2E mobile robot} \\ \midrule
			$2R_{\text{Rob}}$ & 35 cm \\
			$V_{\max}$ & 0.7 m/s \\
			Chosen robot speed $V$ & 0.17 m/s \\
			On-board computer & Raspberry Pi with integrated WiFi \\  \midrule
			\multicolumn{2}{c}{Studio} \\ \midrule
			Workspace volume & $3.5\times3.5\times2$ (in m) \\
			Motion capture & Optitrack Flex 13 localization cameras \\
			\bottomrule     
		\end{tabular}  
\end{table}

Within the Studio, all robots are localised in the same coordinate frame. The position, velocity, and orientation of each robot are determined using the motion capture system and transmitted by WiFi to a central computer, where the collision avoidance algorithm is implemented. The angular speed, $\omega$, which is given by the control law \eqref{omegaRiCL}, of the robot and its constant linear speed, $V$, are obtained by driving the wheels of the robot independently. These states are related by the kinematic relations
\begin{equation}\label{VomegaCal}
    V = \frac{v_R+v_L}{2}, \ \omega = \frac{v_R-v_L}{d}, \ v_{R,L}=\omega_{R,L}r_w,
\end{equation}
where, $v_{R,L}$ are the linear speeds of the right and left wheels; $d$ is the distance between the wheels; and $r_w$ is the radius of each wheel.

In all the experimental results, once the robots reach within 20 cm of their goal locations, they stop. This setting has support in Remark~\ref{AttProofRemark}, since the speeds of the robots are kept constant throughout. Also, for all cases, the paths of the robots in the $X-Y$ space, the trajectories of the relative velocities in the $V_r-V_{\theta}$ space, and the evolution of the separation distance, $r$, are shown.

\subsection{Case 1 - Results}\label{Sec:Case1Res}

Two cooperative robots, $R_{1,2}$, are initially at rest at positions $(x,y)_{R10}$ and $(x,y)_{R20}$, respectively. For each robot, its goal locations is selected as the initial position of the other. They are also oriented such that, if they move at a constant velocity, they are on a head-on collision course. The PF parameters were set at $\lambda=10$ and $\kappa=10$.

\begin{figure}[!htpb]
	\centering
	\includegraphics[width=0.95\columnwidth]{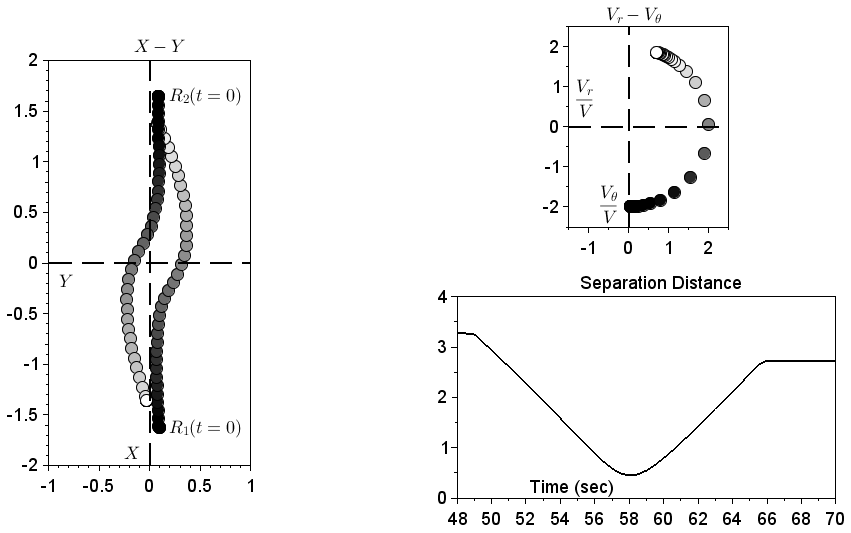}
	\caption{Collision avoidance by a pair of cooperative robots initially on a head-on collision course. As time progresses, the shade of the circles becomes lighter; this visualisation idea is borrowed from \cite{vandenBerg2011}.}
	\label{fig:ClsAvdCoopPair}
\end{figure}

The reciprocal nature of collision avoidance is evident from the paths traced by the robots in the $X-Y$ space (depicted as shaded circles) shown in Fig.~\ref{fig:ClsAvdCoopPair}. Each robot turns to its right in order to avoid collisions; this direction of turn is decided by the choice of the signs of the gradients that define the inputs. As can also be seen in the trajectory of the relative velocities in the $V_r-V_{\theta}$ space, there are no stable equilibria in the $\{V_r<0 \ \text{and} \ V_{\theta}=0\}$ domain, thus ensuring that the conditions for collision are never satisfied. In addition, the robots do not display any oscillatory behaviour as well; this can also be observed in the results that follow. 

To generate the graph in the $V_r-V_{\theta}$ space, the values of these states are divided by the set-speed $V$. Moreover, once the condition $V_r=0$ holds, the triggering condition becomes invalid and the robots deactivate the repulsive fields. They then move to their respective goal locations based on the inputs from the Attractive PF. Since the robots have non-zero radii, the parameter $\lambda$ is tuned during experimentation, so that they avoid collisions. This is equivalent to changing the parameter $r^*$ in the expression of the inputs given by \eqref{FxySatFn}. Setting $\lambda=5$ led to the robots colliding with each other, while for $\lambda\geq10$, they were able to maintain non-zero separation distance between themselves. Thus, for larger $\lambda$, the robots could turn away faster from each other. The expressions \eqref{FxySatFn} and \eqref{lFlimExpNonCoop} support this choice. 

\begin{figure}[!htpb]
	\centering
	\includegraphics[width=0.95\columnwidth]{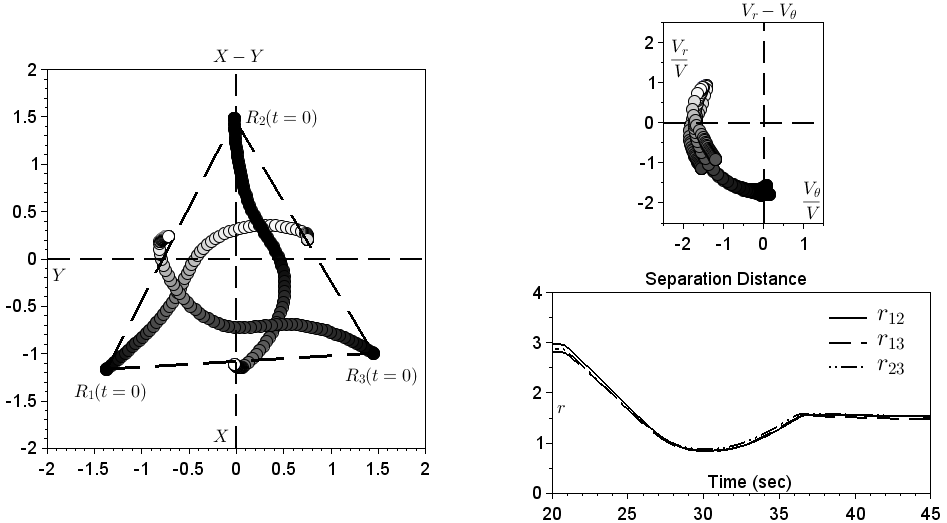}
	\caption{Collision avoidance by three cooperative robots all of which are initially on a collision course.}
	\label{fig:ClsAvdCoopThree}
\end{figure}

The case of 3 cooperative robots is presented next. Each robot is initially located at the vertex of an (almost) equilateral triangle; their goal locations are the mid-points of the sides opposite to their respective vertices. The results are shown in Fig.~\ref{fig:ClsAvdCoopThree} and support the theoretical discussion presented in Sec.~\ref{Sec:MultRob}. The aspect that stands out in this case is the behaviour of the three robots, which follows the ``symmetric roundabout'' behaviour as described in \cite{Kosecka1997} and which can be seen in the figure on the left in Fig.~\ref{fig:ClsAvdCoopThree}. As can be seen, the three robots begin to trace a circle in the clockwise direction leading to this behaviour. This behaviour is visualised in Fig.~\ref{fig:2RobClsSchematicHO}; since this behaviour occurs for each pair, the resulting motion in $X-Y$ space appears as though the robots are following a roundabout.

In fact, the robots considered pair-wise follow the discussion presented in Sec.~\ref{Sec:NonZeroRob}. Each robot begins to trace a circle passing through the vertex of the triangle in the clockwise direction. Thus, if the initial separation and PF parameters are chosen according to the results in Sec.~\ref{Sec:NonZeroRob}, then none of these circles intersects with any other, leading to collision avoidance amongst all robots. It is emphasised that no rules are stated that define the direction of turn of each robot, but appear naturally based on the choices of the signs of the gradients, defined in \eqref{GradURepF}.

\subsection{Case 2 - Results}

The results of a cooperative robot, $R_1$, avoiding a collision with a non-cooperative robot, $R_2$, are shown in Fig.~\ref{fig:ClsAvdNonCoop}. The robots are initially on a head-on collision course. As can be seen, the robot $R_2$ moves towards its goal, which is the initial position of $R_1$, at a constant velocity and does not maneuver even though $R_1$ is on its path. $R_1$, on the other hand, projects the dynamic vortex PF around $R_2$ and maneuvers around it, since the triggering condition becomes true. As can be seen in Fig.~\ref{fig:ClsAvdNonCoop}, the equilibrium conditions, identified in Theorem~\ref{NonCoopRobClsAvdThm}, are not stable.

\begin{figure}[!htpb]
	\centering
	\includegraphics[width=0.95\columnwidth]{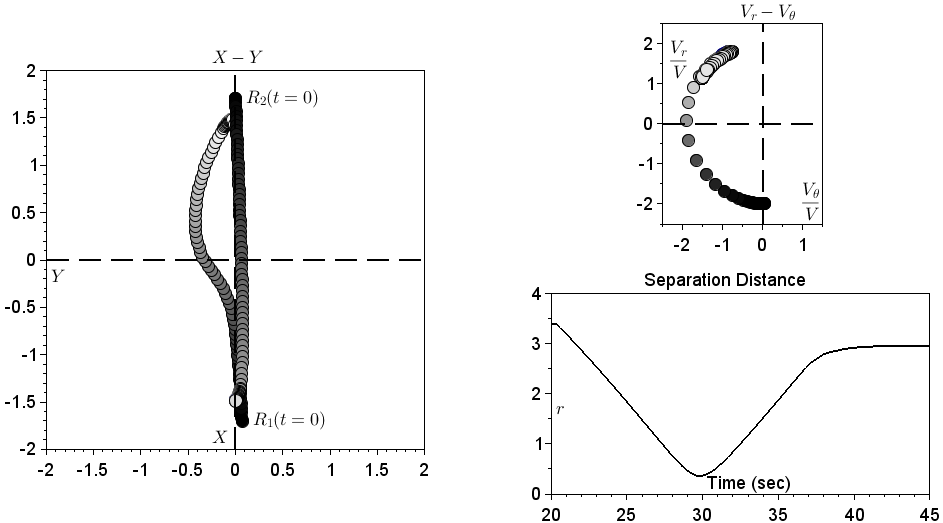}
	\caption{A cooperative robot avoiding collision with a non-cooperative robot.}
	\label{fig:ClsAvdNonCoop}
\end{figure}

\subsection{Case 3 - Results}

The results of a cooperative robot, $R_1$, avoiding a collision with an attacking robot, $R_2$, are shown in Fig.~\ref{fig:ClsAvdAttack}. Even here, the robots are initially on a head-on collision course. The PF parameters are selected such the conditions - identified in Theorem~\ref{AttackRobClsAvdThm} - for $R_1$ to avoid colliding with $R_2$ hold. As can be seen, $R_2$ maneuvers and begins to ``chase'' $R_1$, while $R_1$ has already changed its heading angle in order to avoid colliding with $R_2$. In this experiment, once $R_1$ has reached the vicinity of its goal, it stops moving and only then does $R_2$ collide with it. During experimentation, it was observed that a lower value of the repulsive PF parameter, $\lambda$, resulted in $R_1$ making several ($>2$) circular turns before it reached its goal; however, at no instant was $R_2$ able to collide with $R_1$. This behaviour is equivalent to the oscillatory behaviour recognised in Sec.~\ref{Case3Sec}. 

\begin{figure}[!htpb]
	\centering
	\includegraphics[width=0.95\columnwidth]{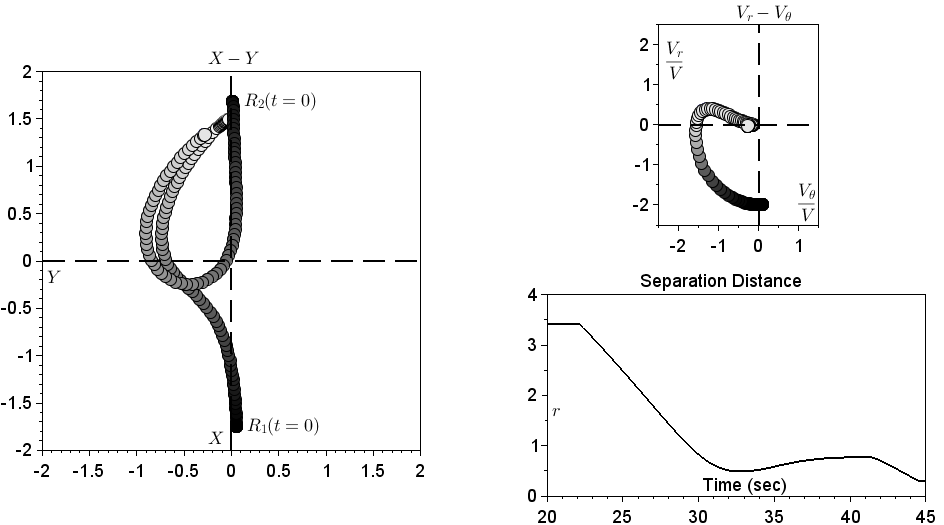}
	\caption{A cooperative robot avoiding collision with an attacking robot.}
	\label{fig:ClsAvdAttack}
\end{figure}

\section{Comparison with a non-vortex PF}

Simulation results and a brief analysis are presented when a pair of cooperative robots apply inputs given by the non-vortex PF, \eqref{GradURepFNonV}. Analogous to the results in Sec.~\ref{Case1Sec}, analysis for this case does not consider the attractive PF. Thus, without the vortex field, from \eqref{VrVthDotExp}, in the closed-loop, the relative accelerations become 
\begin{subequations}\label{VrVthNonVortCL}
    \begin{align}
		\dot{V}_{r} &= \frac{V_{\theta}^2}{r} - 2\left(\frac{\lambda}{V_{rel}r^2}\right)V_{r}^2 \label{VriDotExpCL4} \\
		\dot{V}_{\theta} &=  \frac{-V_{r}V_{\theta}}{r} +4\left(\frac{\lambda}{V_{rel}r^2} \right)V_{r}V_{\theta}.\label{VthiDotExpCL4}
    \end{align}
\end{subequations}
For these dynamics, the equilibrium conditions are the pairs $V_r=V_{\theta}=0$ and $V_r=-2V,V_{\theta}=0$. The latter is an equilibrium condition since the relative velocities satisfy \eqref{VrVthVrelCirc}, thus, even if the magnitude of $\dot{V}_r$ is high, the relative velocity $|V_r|\leq 2V$. The existence of this additional equilibrium point is also different from the closed-loop dynamics \eqref{VrVthDotExpCL}, which results from the use of the vortex PF. 

Lyapunov analysis is performed to analyse the stability of this equilibrium point, since it is the conditions $V_r<0,V_{\theta}=0$ that lead to collision. Thus, if this equilibrium point is unstable, then collision is avoided. From the Lyapunov function
\begin{equation}\label{LyapFnRepNV}
    \mathcal{V}_{LNV}= r + \frac{1}{2}V_{\theta}^2 + \frac{1}{2}\left(V_r+2V\right)^2
\end{equation}
and its derivative
\begin{align}\label{LyapFnRep2CODot}
	\dot{\mathcal{V}}_{LNV} &= V_r\left(1 + \alpha\left(V_{\theta}^2-V_{r}^2\right)\right) + 2V\dot{V}_r \nonumber\\
	                         &= V_r\left(1 + \alpha V^2\right)-2\alpha\left(V_r+2V\right)V_r^2 + 2V\frac{V_{\theta}^2}{r}.
\end{align}
where, $\alpha=\frac{2\lambda}{V_{rel}r^2}>0$. Now, if the robots are initially on a head-on collision course, that is, $0<V_r(t=0)\leq-2V,V_{\theta}(t=0)=0$, then, from \eqref{VthiDotExpCL4}, it can be seen that $\dot{V}_{\theta}=0 \ \forall \ t\geq0$, implying at the robots do not turn away from each other. Thus, for these initial conditions, which are also the equilibrium conditions, the relative velocity $V_r\to-2V$, implying that the derivative $\dot{\mathcal{V}}_{LNV}=V_r\left(1 + \alpha V^2\right)<0$, and in turn, that the equilibrium point is asymptotically stable when the collision conditions are satisfied. Thus, the robots collide with each other. This result is in direct contrast with the use of the vortex PF, where, the closed-loop dynamics have a single equilibrium point $V_r=V_{\theta}=0$ and the inputs derived from the vortex PF ensure that the robots turn away from each other rendering the equilibrium point to be unstable.

Simulation results showing the trajectories of the robots in the $X-Y$ space, with and without the vortex field, are shown in Fig.~\ref{fig:VNoVPFFig}. In each case, the robots are initially on a head-on course with the target of one robot set as the initial position of the other. As is evident, without the vortex field, the robots do not turn and eventually collide with each other, near the origin; for the results in Fig.~\ref{fig:VNoVPFFig}, the simulation was not stopped though the robots collide. The result with the vortex field is the same as discussed in the earlier sections: the robots turn away from each other and then move to their goal locations. 

\begin{figure}[!htpb]
	\centering
	\includegraphics[width=0.8\columnwidth]{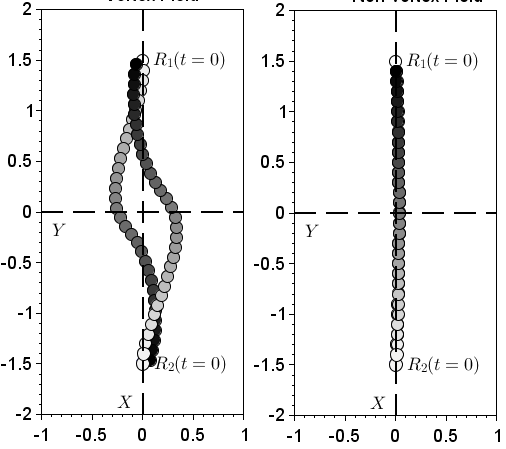}
	\caption{Trajectories of the cooperative robots in $X-Y$ space with (left) and without (right) the vortex repulsive PF.}
	\label{fig:VNoVPFFig}
\end{figure}

\section{Conclusions}\label{Sec:Conc}

In this article, we presented a dynamic vortex PF algorithm for collision avoidance between different types of non-holonomic robots and multiple robots as well. The algorithm is able to ensure that there are no local minima in the relative velocity space when a cooperative robot is on a collision path with other cooperative robots, non-cooperative robots, as well as an attacking robot. For the last case, conditions on the PF parameters were identified such that the attacking robot is never able to collide with a cooperative robot. The nature of the vortex PF guarantees that robots turn in the same direction while avoiding collisions with one another, thus lending the robots a reciprocal behaviour. With the identification of the free space that is needed by a robot with non-zero radius and bounded inputs to avoid collisions, the algorithm can be extended to robots navigating spaces such as corridors or paths where the robot maybe forced to turn in one direction. In such cases, the problem of how to select the direction of rotation of the robot would need to be addressed. This article also considered robots to move at constant linear speeds, thus the algorithm can be extended when this assumption does not hold. 

\section*{Acknowledgments}
This work was supported by Karnataka Innovation \& Technology Society, Dept. of IT, BT and S\&T, Govt. of Karnataka vide GO No. ITD 76 ADM 2017, Bengaluru; Dated 28.02.2018.

\bibliographystyle{unsrt}  
\bibliography{VortexPF_ICARA23}  

\end{document}